\documentclass{article}
\usepackage{amsmath,graphicx}
\usepackage{bm,enumerate,amssymb,amsthm}
\usepackage{listings}

\usepackage{stmaryrd}
\usepackage{booktabs}
\usepackage{mathtools}
\usepackage{url}
\usepackage[]{subfigure}
\newtheorem{prop}{Proposition}

\usepackage[]{color}
\usepackage[top=30truemm,bottom=30truemm,left=25truemm,right=25truemm]{geometry}

\title{Neural Time Warping For Multiple Sequence Alignment}
\author{Keisuke Kawano, Takuro Kutsuna, Satoshi Koide}
\date{}
\begin{document}
%
\maketitle
\begin{center}
Toyota Central R\&D Labs. Inc., Nagakute, Aichi, Japan
\end{center}

\begin{abstract}
  Multiple sequences alignment~(MSA) is a traditional and challenging task for time-series analyses.
  The MSA problem is formulated as a discrete optimization problem and is typically solved by dynamic programming.
  However, the computational complexity increases exponentially with respect to the number of input sequences.
  In this paper, we propose \emph{neural time warping~(NTW)} that relaxes the original MSA to a continuous optimization and obtains the alignments using a neural network.
  The solution obtained by NTW is guaranteed to be a feasible solution for the original discrete optimization problem under mild conditions.
  Our experimental results show that NTW successfully aligns a hundred time-series and significantly outperforms existing methods for solving the MSA problem.
  In addition, we show a method for obtaining average time-series data as one of applications of NTW.
  Compared to the existing barycenters, the mean time series data retains the features of the input time-series data.
\end{abstract}

\section{Introduction}
\label{sec:intro}

Time-series alignment is one of the most fundamental operations in processing sequential data, and has various applications including speech recognition~\cite{shimodaira2002dynamic}, computer vision~\cite{chang2019discriminative, trigeorgis2016deep}, and human activity analysis~\cite{sheikh2005exploring,gritai2009matching}.
Alignment can be written as a discrete optimization problem that determines how to warp each sequence to match the timing of patterns in the sequences.
Dynamic time warping~(DTW)~\cite{sakoe1978dynamic} is a prominent approach, which finds the optimal alignment using the dynamic programming technique.
However, for \textit{multiple sequence alignment~(MSA)}, it is difficult to utilize DTW because the computational complexity increases exponentially with the number of time-series data.
In this paper, we propose a new algorithm, \textit{neural time warping~(NTW)} to solve MSA, which can even be utilized for one hundred time-series data.

An important idea for MSA is to relax the discrete optimization problem to a continuous optimization problem with the warpings modeled by combining basis functions~\cite{zhou2012generalized, khorram2019trainable}.
Generalized time warping~(GTW)~\cite {zhou2012generalized} has been proposed to express warpings by a linear combination of predefined basis functions.
However, GTW requires a large number of basis functions to express complex warpings.
Trainable time warping~(TTW)~\cite{khorram2019trainable} models the warpings using discrete sine transform~(DST), i.e., sine functions as the basis functions.
TTW controls the non-linearity of the warpings by limiting the number of DST coefficients to avoid sudden changes in the warpings.
The solution of the original problem can be obtained by converting the solution of the relaxed continuous optimization.
However, the converted warpings of the existing methods may not be valid solutions for the original MSA problems.
In contrast, NTW guarantees the validity; as such, NTW has the ability to express flexible warpings.

Neural networks have been successfully applied to various types of applications, and recently some neural network-based alignment methods have been proposed~\cite{Dogan18neumatch, grabocka2018neuralwarp}.
NeuMATCH~\cite {Dogan18neumatch} is a model that outputs alignments for the matching of videos and texts. 
NeuralWarp~\cite{grabocka2018neuralwarp} predicts whether or not to align frames of the sequence using embedded features. 
However, these methods are supervised learning-based approaches, i.e., they require many true alignments, while NTW aligns the sequences optimizing a neural network in an unsupervised manner for each MSA problem.

One of applications of time-series alignment is to find a representative time-series from a large number of time-series data.
DTW Barycenter Averaging~(DBA)~\cite{Petitjean2011-DBA} provides DTW barycenters, which minimizes the sum of the DTW discrepancies to the input time-series.
DBA repeats (1) aligning each time-series to the current barycenter sequence and (2) updating the barycenter alternately.
Unlike DTW barycenters, we can easily obtain statistics (e.g., averages and standard deviation) of the aligned time-series using NTW.

\textbf{Contributions:}
For the executable and accurate MTSAs, we propose NTW, which models flexible warpings with a neural network, thereby guaranteeing the feasibility after discretization~(Section~\ref{sec:ntw}).
To mitigate poor local minima problems that often occur in the gradient-based optimization of MSA problems, we propose an annealing method (Section~\ref{sec:annealing-method}), which aligns only the low-frequency parts in the early steps of the optimization, and then gradually aligns the high-frequency parts.
Furthermore, we demonstrate the advantage of NTW against existing methods using UCR time series classification archive~\cite{UCRArchive}~(Section~\ref{sec:experiments}).
Lastly, we show some examples of averaged time-series data obtained from aligned sequences~(Section~\ref{sec:average}).

\textbf{Notations:}
We denote $\{0, \dots, N\}$ as $\llbracket N \rrbracket$.
$\boldsymbol{1}_N \in \mathbb R^N$ is a $N$ dimensional vector of ones.

\section{Neural Time Warping}
\label{sec:ntw}

\subsection{Problem setting}
\label{sec:problem-setting}

Given $N$ time-series, $x_i=[x_i^{(0)}, \dots , x_i^{(T_i)}] \in \mathbb R^{T_i+1}, i=1,\dots, N$,
the MSA problem can be formulated as the optimization problem, which minimizes the sum of distances between warped time-series:
\begin{align}
  \label{eq:problem}
  \min_{\tau\in \mathcal T} \sum_{z=0}^{Z} \sum_{i=1}^{N} \sum_{j=1}^N d(\hat{x}_i(z; \tau_i), \hat{x}_j(z; \tau_j)).
\end{align}
Here, $\mathcal T$ is a set of warpings $\tau=[\tau_1,\dots, \tau_N]$, which satisfies the following constraints for all $\tau_i: \llbracket Z \rrbracket \mapsto \llbracket T_i \rrbracket$.
\begin{align}
  \label{eq:monotonicity}
    & \tau_i(z) \leq \tau_i(z+1),    \ \ \ \ \ \ \ \     \text{(monotonically nondecreasing)}     \\
  \label{eq:continuity}
    & \tau_i(z+1) - \tau_i(z) \leq 1,\     \text{(continuity)}         \\
  \label{eq:boundary}
    & \tau_i(0)=0, \tau_i(Z)=T_i.    \    \text{(boundary conditions)}
\end{align}
$\hat{x_i}(z; \tau_i) \coloneqq \sum_{t=0}^{T_i} x_i^{(t)} \delta(t - \tau_i(z))$ denotes a function~($\llbracket Z \rrbracket \mapsto \mathbb R$) corresponding to a warped time-series of $x_i$ by $\tau_i$, where $\delta$ is the delta function.
$d$ is a distance function between frames~(e.g., Euclidean distance) and the length of the warped time-series is $Z+1$.
Note that we can extend the alignment problem to multi-dimensional time-series using an appropriate distance function $d$~\cite{trigeorgis2016deep}.
DTW~\cite{sakoe1978dynamic} obtains the optimal $\tau$ and $Z$ with the dynamic programming technique.
However, the computational complexity of DTW is $\mathcal O(\prod_{i=1}^N T_i)$, thereby making it difficult to apply DTW to the MSA problem.

\subsection{Continuous warpings}
\label{sec:warp-cont-doma}
To reduce the computational complexity, NTW relaxes the discrete optimization problem to a continuous optimization problem, and searches the optimal \emph{continuous warpings} $\tau'(s) = [\tau'_{1}(s),\dots,\tau'_{N}(s)]:[0, S]\mapsto [0, 1]^N$ instead of~$\tau$.
Note that any positive value can be assigned to $S$ without a loss of generality.
First, we interpolate the input time-series with sinc interpolation:
\begin{align}
  x'_i(t') \coloneqq \sum_{t=0}^{T_i} x_i^{(t)} \text{sinc} \left(t - t'T_i \right),
  \text{sinc}(t) \coloneqq \begin{cases}
    \frac{\text{sin}(\pi t)}{t} & t \neq 0\\
    1 & t=0 ,
  \end{cases} \nonumber
\end{align}
where we assume $t' \in [0, 1]$.
With a continuous warping~$\tau'_i(s)$, the warped time-series can be written as $x'_i(\tau'_{i}(s) )$~\cite{khorram2019trainable}. 
For clarity, we denote the warpings of the original MSA problem as \emph{discrete warpings}.


Given a continuous warping $\tau'$, a discrete warping $\tilde{\tau}_i$ of length~$Z+1$ can be calculated in the following way:
(1) prepare monotonically increasing points $[s_0, \ldots, s_Z]$ such that $s_z \in [0, S]$, $s_0=0$, and $s_Z=S$,
(2) calculate the discrete warping $\tilde{\tau}_i$ for $i=1,\ldots, N$ by $\tilde{\tau}_i(z) \coloneqq \lfloor  T_i \tau'_i(s_z)  \rfloor,$ where $\lfloor \cdot \rfloor$ is a floor function.
In the following sections, we assume $s_z$ to be at regular intervals in $[0,S]$, i.e., $s_z = \frac{z}{Z}S$ for $z \in \llbracket Z \rrbracket$.
A discrete warping obtained from a continuous warping by the above procedure is called a \emph{sampled warping}.

\subsection{Modeling of the warpings}
\label{sec:modeling-warpings}
It is necessary to model the warpings~$\tau'$ with a flexible function that can operate even when optimal warpings have complex forms.
However, it is difficult to guarantee that the sampled warpings satisfy the alignment constraints~(\ref{eq:monotonicity}--\ref{eq:boundary}) with the flexible function.
NTW carefully models $\tau'_\theta$ using the orthogonal basis $\boldsymbol{e}_k\in \mathbb R^N, k=1, \dots, N$, and a non-linear function $\phi_\theta (s): \mathbb [0, S] \mapsto \mathbb R^{N-1}$ with its parameters $\theta$ as follows.
\begin{align}
  \label{eq:warp_model}
  \tau'_{\theta}(s) \coloneqq s \boldsymbol{e}_1 + s\left(S - s\right) \sum_{k=1}^{N-1}[\phi_{\theta} (s)]_k \boldsymbol{e}_{k+1},
\end{align} where $\boldsymbol{e}_1 = \frac{1}{\sqrt{N}} \boldsymbol{1}_N$.
The following proposition holds for any functions~$\phi_\theta$ in Eq.\eqref{eq:warp_model}.
\begin{prop}
  \label{prop:boundary}
  If we set $S=\sqrt N$, the sampled warpings~$\tilde{\tau}_{i}$ obtained from $\tau'_\theta$ in Eq.\eqref{eq:warp_model} satisfies the boundary condition~\eqref{eq:boundary}. 
\end{prop}
In the followings, we assume $S=\sqrt N$. Then the following proposition holds:
\begin{prop}
  \label{prop:1}
  If $\tau'_{\theta,i}(s_{z+1})-\tau'_{\theta,i}(s_z)\geq 0$ $(\forall i \in \{1,\dots,N\}$, $\forall z\in \llbracket Z \rrbracket)$ and $Z \geq N \max_{j=1}^N T_j$,
  then the sampled warping $\tilde\tau$ obtained from $\tau'_\theta$ in Eq.\eqref{eq:warp_model} satisfies the monotonicity~\eqref{eq:monotonicity} and the continuity~\eqref{eq:continuity}.
\end{prop}
\begin{proof}
  From the first assumption, the sampled warpings obviously satisfy the monotonicity~\eqref{eq:monotonicity}.
  By considering the regular intervals $s_{z+1}- s_z= \frac{\sqrt N}{Z}$, we have
    \begin{align}
        \label{eq:diff}
      &\sum_{i=1}^N \tau'_{\theta,i}(s_{z+1}) - \tau'_{\theta,i}(s_{z})  \nonumber\\
      =&\sum_{i=1}^N \frac{\sqrt N}{Z} \boldsymbol{e}_{1,i} + \sum_{k=1}^{N-1} \big[s_{z+1}(\sqrt N - s_{z+1})\phi_\theta(s_{z+1}) \nonumber \\
        &\ \ \ \ \ \ \ \ \ \  \ \ \ \ \ \ \ \ \ \ \ \ \  -  s_z(\sqrt N - s_{z})\phi_\theta(s_z) \big]_k \sum_{i=1}^N\boldsymbol{e}_{k+1, i}.
    \end{align}
  We also have $\sum_{i=1}^N \boldsymbol{e}_{k+1, i}= 0, \forall k\in \{1,\dots, N-1\}$ because $\boldsymbol{e}_{k+1}$ are orthogonal to $\boldsymbol{e}_1=\frac{1}{\sqrt N} \boldsymbol{1}_N$.
  Therefore, we have $\sum_{i=1}^N \tau'_{\theta,i}(s_{z+1}) - \tau'_{\theta,i}(s_{z}) =  \frac{N}{Z}$.
  Furthermore, from $\tau'_{\theta, i}(s_{z+1}) - \tau'_{\theta, i}(s_z)\geq 0$, $\tau'_{\theta,i}(s_{z+1}) - \tau'_{\theta,i}(s_{z}) \leq \frac{N}{Z}$ holds for all $i$.
  By multiplying $T_i$, for all $i$, we obtain
  \begin{align}
    \begin{split}
     T_i (  \tau'_{\theta, i}(s_{z+1}) - \tau'_{\theta, i}(s_z) ) \leq \frac{T_i N}{Z}  \leq \frac{T_i}{\max_{j=1}^N T_j} \leq  1.
    \end{split}
  \end{align}
  Because $a-b \leq 1 \Rightarrow \lfloor a \rfloor - \lfloor b \rfloor \leq 1$ holds in general, we obtain the continuity~\eqref{eq:continuity} as
  \begin{align}
  \tilde\tau_{i}(s_{z+1})-\tilde\tau_{i}(s_z) = \lfloor T_i\tau'_{\theta, i}(s_{z+1}) \rfloor -  \lfloor T_i\tau'_{\theta, i}(s_{z}) \rfloor \leq  1. \nonumber
\ \ \ \ \ \     \qedhere
  \end{align}
\end{proof}
Instead of Eq.\eqref{eq:warp_model},  if we employ another model, e.g.,
$$\tau'_{\theta}(s) = s\left(\sqrt N - s\right) \sum_{k=1}^{N}[\psi (s)]_k \boldsymbol{e}_{k} + \frac{s}{S}$$  with $\psi: \mathbb R \mapsto \mathbb R^N$,
  the sampled warpings satisfy the boundary condition~\eqref{eq:boundary}, but not necessarily the continuity~\eqref{eq:continuity}.
Similar to the alternative model, a discrete warping obtained with TTW~\cite{khorram2019trainable} also satisfies the boundary condition~\eqref{eq:boundary} and the monotonicity~\eqref{eq:monotonicity}, but does not always satisfy the continuity~\eqref{eq:continuity}.
A discrete alignment obtained with GTW~\cite{zhou2012generalized} does not always satisfy the boundary condition~\eqref{eq:boundary} and the continuity~\eqref{eq:continuity}.
In contrast, NTW guarantees that the obtained discrete alignment satisfies the conditions~(\ref{eq:monotonicity}--\ref{eq:boundary}) as long as $\tau'_\theta$ satisfies the condition described in Proposition 2. In NTW, a neural network is employed as the non-linear functions for the warpings.
The network parameters can be learned with end-to-end gradient-based optimization.

\subsection{Optimization}

From Proposition~\ref{prop:1}, we consider the following continuous optimization problem for NTW.
\label{sec:optimization}
 \begin{align}
   \label{eq:lossfunc_org}
   \begin{split}
     \min_{\theta} \int_{s=0}^{\sqrt N} \sum_{i=1}^{N} \sum_{j=1}^N (x'_i(\tau'_{\theta,i}(s))- x'_j(\tau'_{\theta,j}(s)))^2 ds, \\
     \text{s.t.,\ } \tau'_{\theta,i}(s_{z+1}) \geq \tau'_{\theta,i}(s_z), (\forall i\in \{1,..., N\}, \forall z \in \llbracket Z \rrbracket).
   \end{split}
 \end{align}
 To utilize recent optimization techniques for neural networks~\cite{kingma2015adam}, we employ the following unconstrained optimization problem with a penalty term $r$.
\begin{align}
  \label{eq:lossfunc}
  &\min_{\theta} \int_{s=0}^{\sqrt{N}}\sum_{i=1}^{N} \sum_{j=1}^N (x'_i(\tau'_{\theta,i}(s))- x'_j(\tau'_{\theta,j}(s)))^2ds + \lambda  r(\tau'_\theta),   \nonumber \\
  &r(\tau'_\theta) \coloneqq \sum_{i=1}^{N} \sum_{z=0}^{Z-1} \max \left( \tau'_{\theta,i}(s_{z+1}) -  \tau'_{\theta,i}(s_{z}) , 0 \right),
\end{align}
where $\lambda$ is a hyperparameter for the penalty term.
Using a large enough $\lambda$, the solutions satisfying the monotonicity on $s_z$ can be obtained.
We approximate the integration by the trapezoidal rule because $s$ is a one dimensional variable.
Note that Proposition 2 hold as long as $r(\tau'_\theta)=0$.
Finally, we obtain the following objective function for NTW,
 \begin{align}
   \label{eq:lossfunc2}
   &\min_{\theta} \sum_{z'=0}^{Z'-1}\sum_{i=1}^{N} \sum_{j=1}^N (x'_i(\tau'_{\theta,i}(s_{z'+1}))- x'_j(\tau'_{\theta,j}(s_{z'})))^2ds + \lambda  r(\tau'_\theta),   \nonumber \\
   &r(\tau'_\theta) \coloneqq \sum_{i=1}^{N} \sum_{z=0}^{Z-1} \max \left( \tau'_{\theta,i}(s_{z+1}) -  \tau'_{\theta,i}(s_{z}) , 0 \right),
 \end{align}
 where $z'$ is the sampling points for the trapezoidal rule.

The objective function is calculated by GPU parallel computing.
Moreover, it can be easily implemented with modules for tensor operations~(e.g., PyTorch~\cite{paszke2017automatic}) and the gradients of the parameters~$\theta$ can be obtained by automatic differentiation~\cite{baydin2018automatic}.
When GPU memory is not enough, mini-batch optimization can be used for NTW.

\subsection{Annealing method}
\label{sec:annealing-method}
A further problem of the optimization for NTW is the presence of many poor local minima, which correspond to align only part of the high-frequency components of the input time-series.
To mitigate this problem, we extend the \emph{sinc} function with an annealing parameter $\alpha$ as
\begin{align}
  \text{sinc}'(t; \alpha) = \begin{cases}
    \frac{\text{sin}(\alpha \pi t)}{\alpha t} & t \neq 0\\
    1 & t=0.
  \end{cases}
\end{align}
When $\alpha=1$, $\text{sinc}'$ is equivalent to the original sinc function.
Figure~~\ref{fig:annealing} shows an example of interpolated time-series varying the annealing parameter $\alpha$.
As shown in Figure~~\ref{fig:annealing}, when $\alpha \gg 1$, the interpolated time-series represent only the low-frequency components of the original time-series.
By gradually annealing $\alpha$ from a sufficiently large value to 1, the optimization first aligns the low-frequency components of the time-series and then gradually aligns the high-frequency components.
To demonstrate the annealing technique, Figure~~\ref{fig:annealing_loss} shows an example of loss histories (i.e., the values of the objective function in the problem~\eqref{eq:lossfunc_org}) during the optimizations.
For the ``w/o annealing'' setting, the annealing parameter $\alpha$ is fixed at 1 during the optimization, while we initialized $\alpha$ to 100 then multiplied 0.99 in each update for the ``with annealing'' setting.
Details of the implementation are presented in Section~\ref{sec:experiments}.
As the result shows, NTW without the annealing easily falls into local minima, while NTW with the annealing converges to better solutions.

\begin{figure}[t]
  \centering
  \subfigure[Interpolated time-series varying annealing parameters (50words dataset, label 33~\cite{UCRArchive}).]{
    \label{fig:annealing}
    \includegraphics[width=0.3\linewidth]{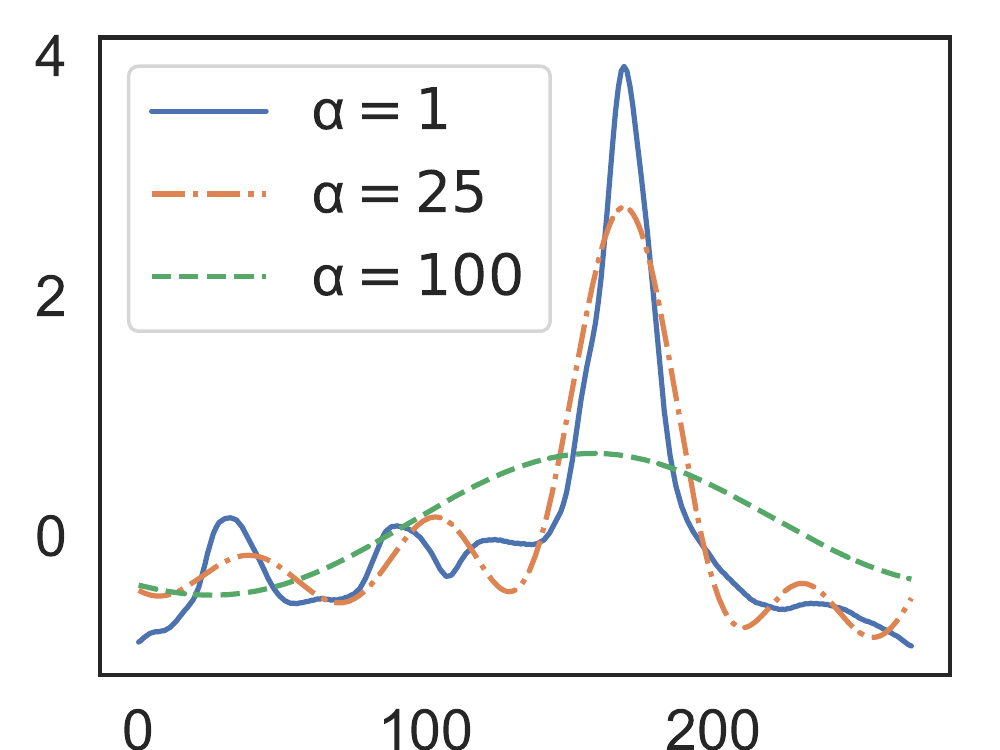}
    }\hspace{1.2mm}
  \subfigure[Loss histories of NTW, with and without annealing technique.]{
    \label{fig:annealing_loss}
    \includegraphics[width=0.3\linewidth]{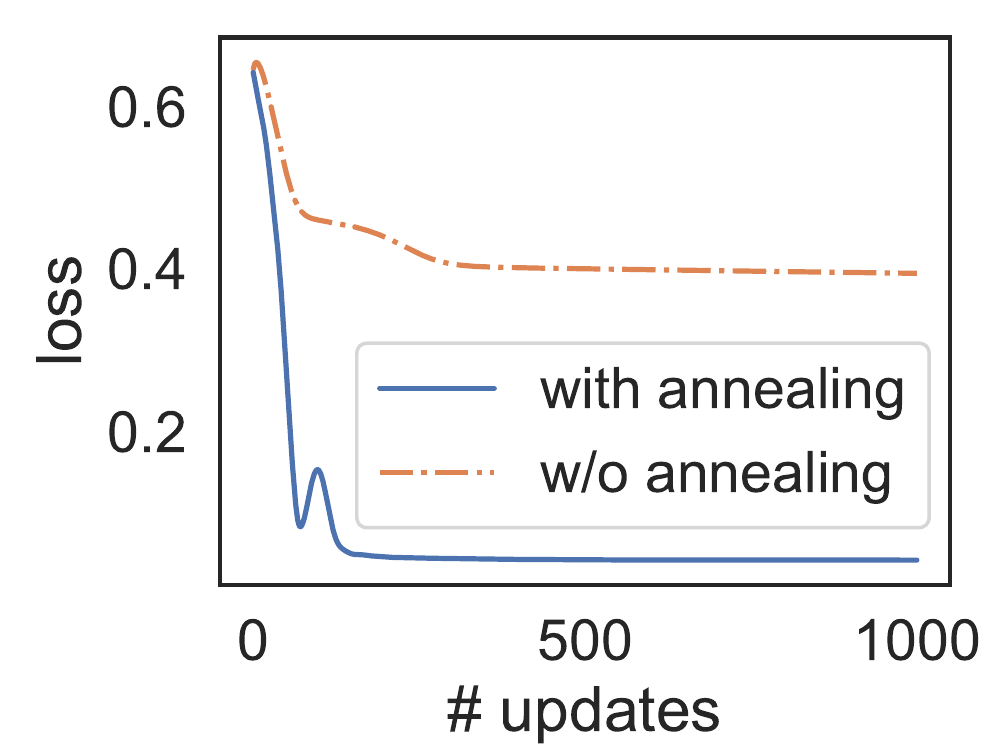}
  }
  \caption{Examples for annealing method}
\end{figure}

\subsection{Warped Average of Time-Series Data}
Given $N$ aligned time-series, $x'_i\in \mathbb R^{Z}, i=1,\dots,N$, we can easily obtain statistics of the aligned time-series.
For example, the average time-series can be obtained as $\bar x \coloneqq \frac{1}{N} \sum_{i=1}^N x'_i$.
Similarly, the standard deviation of the time-series data is $x_\text{SD}(s) \coloneqq \sqrt{\frac{1}{N} \sum_{i=1}^N \left( x'_i(\tilde \tau_i(s)) - \bar x(s) \right)^2  }$.
We call $\bar x(s)$ and $x_\text{SD}(s)$ warped average and warped standard deviation, respectively.

\section{Experiments}
\label{sec:experiments}

\subsection{Dataset}
\label{sec:dataset-1}
In this section, we compare NTW with existing methods for MSA, namely GTW~\cite{zhou2012generalized} and TTW~\cite{khorram2019trainable} using the UCR time series classification archive (2015)~\cite{UCRArchive}.
We also employ UCR time series classification archive (2015) for comparing warped averages acquired by NTW with barycenters.

This archive includes 85 datasets of time series, such as human actions and medical imaging.
One dataset contains up to 60 class labels and is divided into train and test parts.
For the evaluation, we use all training splits of each class in each dataset if they have more than one time-series data~(642 sets in total).
Due to limited computational resources, we randomly sampled 100 time-series from a set containing more than 100 time-series data.

\subsection{Multiple Sequences Alignment}
\textbf{Evaluation metric: }
We cannot evaluate the accuracies of MSA directly, because it is difficult to obtain the true alignments for a large number of time-series data.
Following~\cite{khorram2019trainable}, we employed the barycenter loss;
\begin{align}
\mathcal L_b \coloneqq \frac{1}{NZ} \sum_{i=1}^N \mathcal D_{\text{DTW}} \left(x_i, \frac{1}{N}\sum_{j=1}^N x'_j(\tilde{\tau}_ {j}) \right),
\end{align}
where $\mathcal D_{\text{DTW}}$ denotes the DTW discrepancy~\cite{sakoe1978dynamic}.

To compare validities of the sampled warpings, we also employ the following validity scores, which measure the ratios of frames satisfying the constraints of MSA~(\ref{eq:monotonicity}--\ref{eq:boundary}), respectively;
\begin{align}
  V_\text{mono} &\coloneqq \frac{1}{N Z} \sum_{z=0}^{Z-1} \sum_{i=1} ^N  I( \tilde{\tau}_i(z+1) \geq \tilde{\tau}_i(z)) \\
  V_\text{cont} &\coloneqq  \frac{1}{N Z} \sum_{z=0}^{Z-1} \sum_{i=1} ^N I( \tilde{\tau}_i(z+1) - \tilde{\tau}_i(z) \leq 1)\\
  V_\text{bound}&\coloneqq \frac{1}{2} \sum_{i=1}^N\left(  I(\tilde{\tau}_i(0) = 0) +  I(\tilde{\tau}_i(Z) = T_i) \right),
\end{align} where $I(cond)$ is 1 when $cond$ is true, and 0 otherwise.

\textbf{Implementations: }
We implemented NTW using PyTorch (version 1.2)~\cite{paszke2017automatic} and experimented with GPUs (NVIDIA TITAN RTX) (See Appendix A for more details).
The neural network is a four-layered, fully connected network with ReLU activations~\cite{nair2010rectified} and skip connections~\cite{huang2017densely}, where the number of the nodes are 1-512-512-1025-[$N-1$], and all weights and biases are initialized to zero.
Via this initialization, NTW initializes the warpings uniformly (i.e., all frames are uniformly warped to align the lengths of the time-series).
See Section~\ref{sec:nnan} for more details.
We updated the parameter 1000 times using Adam~\cite{kingma2015adam} with its default parameters in the PyTorch~\cite{paszke2017automatic} implementation except for the learning rate of 0.0001.
The penalty parameter is set to $\lambda = 1000$, with which the solutions of NTW almost satisfy the monotonicity condition as shown in the following section.
The number of sampling points for the trapezoidal rule is the same as the length of the input time-series.
The annealing parameter $\alpha$ is initialized to 100 and multiplied by 0.99 after every update while it is larger than 1.
For more details, see Section~\ref{sec:pytorchimp}.

For TTW~\cite{khorram2019trainable} and GTW~\cite{zhou2012generalized}, we employed MATLAB implementations provided by the corresponding authors\footnote{TTW: \url{https://github.com/soheil-khorram/TTW}, \\ GTW: \url{https://github.com/LaikinasAkauntas/ctw}}.
As it was reported in~\cite{khorram2019trainable} that the performance of TTW varies depending on the number of DST coefficients~$K$, we employed two models for TTW~($K = 4$ and $16$).

\textbf{Results: }
Table~\ref{tab:barycenter_results} and Figure~~\ref{fig:msa} show some examples of MSA results and average values of barycenter losses, respectively; we chose better results of TTWs ($K=4$ and $K=16$) in each set for TTW (best $K$).
We conducted a pair-wise t-test~(one-sided) over the 642 results.
As shown in the Table~\ref{tab:barycenter_results}, NTW significantly outperformed both TTW and GTW (p-value$<$0.02) on the barycenter losses.
The examples of MSA in Fig~\ref{fig:msa} indicate that NTW has enough flexibility to approximate the complex warpings even when the number of time-series is 100~(uWaveGestureLibrary\_Y~(8) and yoga~(2)).
Table~\ref{tab:valid_ratio} shows the average values of the validity scores, which indicate the results of NTW satisfied the continuity~\eqref{eq:continuity} and boundary conditions~\eqref{eq:boundary} for all MSAs.
Note that the results of NTW rarely~(0.0003\% on average) broke the monotonicity because the unconstrained optimization~\eqref{eq:lossfunc} was used instead of constrained optimization~\eqref{eq:lossfunc_org}.

\begin{table}[t]
  \centering
  \caption{Average barycenter losses  $\mathcal L_b$ ($10^{-2}$).}
  \begin{tabular}{ccccccc}
                               \textbf{NTW}    & TTW (best $K$)  & TTW(4)    & TTW(16)    & GTW    \\ \toprule
                               \textbf{1.515}  &  1.645         &  1.715    & 1.812     & 1.600 \\ 
  \end{tabular}
  \label{tab:barycenter_results}
\end{table}

\begin{table}[t]
  \centering
  \caption{Average validity scores (\%)}
  \begin{tabular}{cccccc}
                            & \textbf{NTW}              & TTW(4)    & TTW(16)    & GTW    \\  \toprule
     $V_\text{mono}$  & 100.0 & 100.0 & 100.0 & 100.0  & \\
     $V_\text{cont}$    & 100.0 & 91.3  & 85.7  & 95.8 & \\
     $V_\text{bound}$       & 100.0 & 100.0 & 100.0 & 84.8  &
  \end{tabular}
  \label{tab:valid_ratio}
\end{table}

\begin{figure*}[t]
  \centering
  \includegraphics[width=0.97\linewidth]{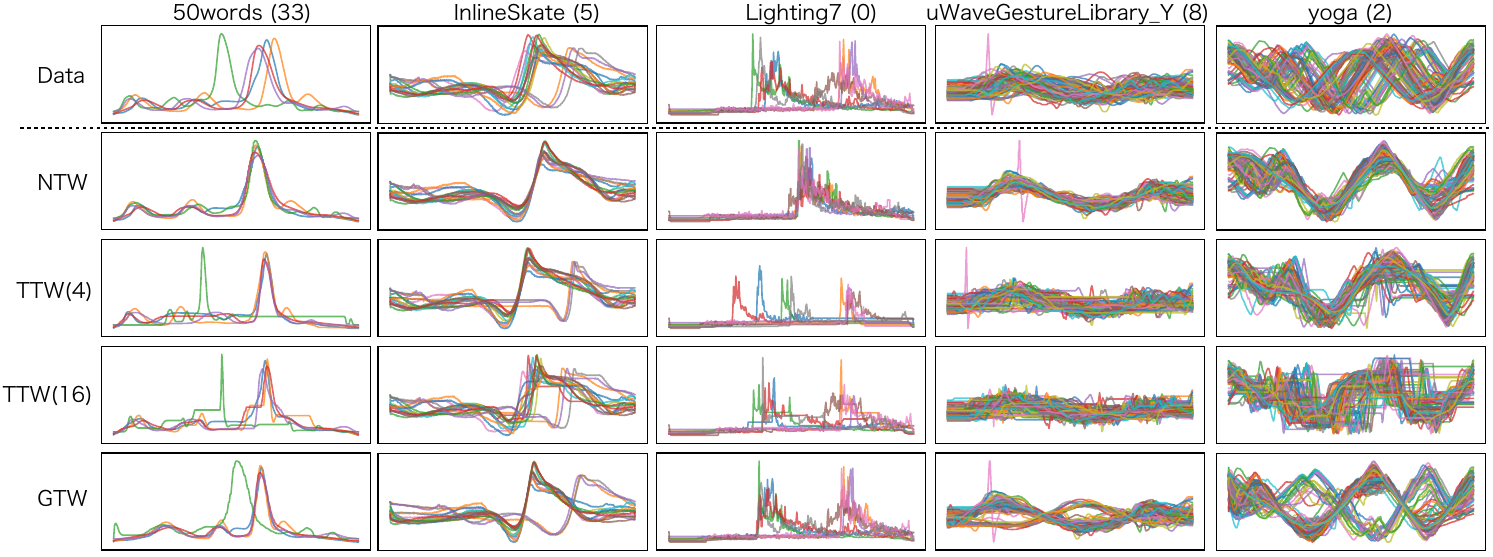}
  \caption{Examples of multiple sequence alignment results. Titles are dataset names and labels.}
  \label{fig:msa}
\end{figure*}

\subsection{Warped Average and Barycenters}
\label{sec:average}
We then visually compare the warped averages acquired by MTSAs and the DTW barycenter.
To calculate the warped average of time-series data for NTW, TTW and GTW, we utilize the aligned time-series data obtained in the previous section.
For DBA~\cite{Petitjean2011-DBA}, we employ the tslearn implementation~\cite{tslearn}.

In Figure~~\ref{fig:averageandbarycenters}, some examples of warped averages obtained by NTW, TTW and GTW and DTW barycenters are shown.
As the figure indicates, the warped averages obtained by NTW retain the characteristics of the input time-series data.
Although the DTW discrepancy between the input time-series data and the barycenters are small, the DTW barycenters obtained by DBA have the sharp peaks that do not exist in the input time-series data.
Note that the DTW discrepancies between the DTW barycenters and the input time-series data are smaller than the DTW discrepancies between the warped averages and the input time-series data.
This is because the sharp peaks can be warped to be matched to the input time-series data when calculating the DTW discrepancies.
The average time-series obtained with TTW and GTW could not represent the characteristics of the input time-series data for some datasets (e.g., Lighting7 (0)) because of the failure of the alignments.

\begin{figure}[t]
  \centering
  \includegraphics[width=0.97\linewidth]{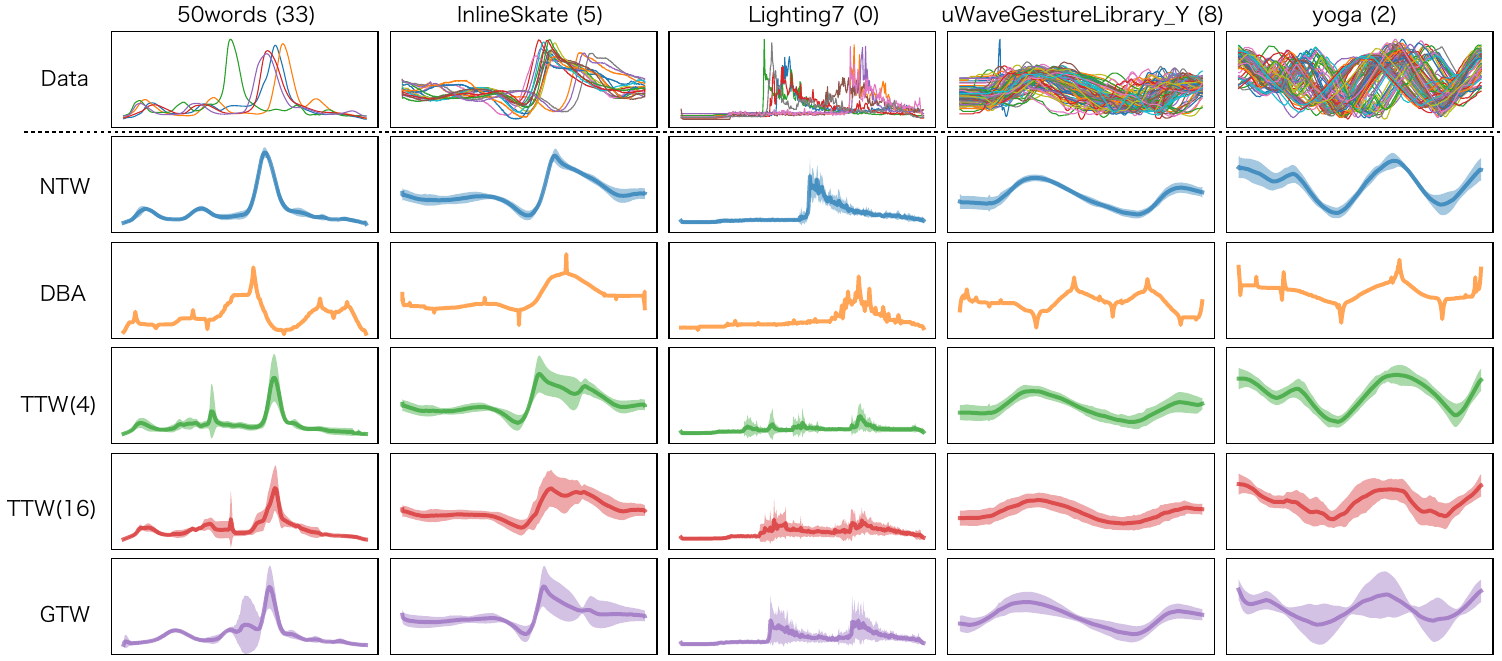}
  \caption{Examples of warped average and barycenters. Titles are dataset names and labels. The lines represent the warped averages (blue: NTW, green and red: TTW, and purple:GTW) and DTW barycenter (orange:DBA). The area denotes warped standard deviations. Note that we cannot obtain warped standard deviations for DTW barycenters.}
  \label{fig:averageandbarycenters}
\end{figure}

\section{Conclusion}
\label{sec:conclusion}

We proposed a new method for MSA, neural time warping~(NTW), which enabled alignments for as many as 100 time-series.
NTW modeled the warpings as outputs of a neural network optimized for each MSA problem.
The solution of NTW is guaranteed to be a valid solution to the original discrete problem under mild conditions.
We also proposed the annealing method to mitigate the problem of poor local minima.
Experiments using UCR time series classification archive showed that NTW outperformed the existing multiple alignment methods.
The warped averages obtained by NTW retain the characteristics of the input time-series data.

\pagebreak
\bibliographystyle{IEEEbib}
\bibliography{references}

\begin{thebibliography}{10}

\bibitem{shimodaira2002dynamic}
Hiroshi Shimodaira, Ken-ichi Noma, Mitsuru Nakai, and Shigeki Sagayama,
\newblock ``Dynamic time-alignment kernel in support vector machine,''
\newblock in {\em Advances in neural information processing systems~(Neurips)},
  2002, pp. 921--928.

\bibitem{chang2019discriminative}
Chien-Yi Chang, De-An Huang, Yanan Sui, Li~Fei-Fei, and Juan~Carlos Niebles,
\newblock ``{$\text{D}^3$TW}: Discriminative differentiable dynamic time
  warping for weakly supervised action alignment and segmentation,''
\newblock in {\em IEEE Conference on Computer Vision and Pattern
  Recognition~(CVPR)}, June 2019.

\bibitem{trigeorgis2016deep}
George Trigeorgis, Mihalis~A Nicolaou, Stefanos Zafeiriou, and Bjorn~W
  Schuller,
\newblock ``Deep canonical time warping,''
\newblock in {\em IEEE Conference on Computer Vision and Pattern
  Recognition~(CVPR)}, 2016, pp. 5110--5118.

\bibitem{sheikh2005exploring}
Yaser Sheikh, Mumtaz Sheikh, and Mubarak Shah,
\newblock ``Exploring the space of a human action,''
\newblock in {\em International Conference on Computer Vision~(ICCV)}. IEEE,
  2005, vol.~1, pp. 144--149.

\bibitem{gritai2009matching}
Alexei Gritai, Yaser Sheikh, Cen Rao, and Mubarak Shah,
\newblock ``Matching trajectories of anatomical landmarks under viewpoint,
  anthropometric and temporal transforms,''
\newblock {\em International journal of computer vision}, vol. 84, no. 3, pp.
  325--343, 2009.

\bibitem{sakoe1978dynamic}
Hiroaki Sakoe and Seibi Chiba,
\newblock ``Dynamic programming algorithm optimization for spoken word
  recognition,''
\newblock 1978, vol.~26, pp. 43--49, IEEE.

\bibitem{zhou2012generalized}
Feng Zhou and Fernando De~la Torre,
\newblock ``Generalized time warping for multi-modal alignment of human
  motion,''
\newblock in {\em IEEE Conference on Computer Vision and Pattern
  Recognition~(CVPR)}. IEEE, 2012, pp. 1282--1289.

\bibitem{khorram2019trainable}
Soheil Khorram, Melvin~G McInnis, and Emily~Mower Provost,
\newblock ``Trainable time warping: Aligning time-series in the continuous-time
  domain,''
\newblock in {\em IEEE International Conference on Acoustics, Speech and Signal
  Processing~(ICASSP)}. IEEE, 2019, pp. 3502--3506.

\bibitem{Dogan18neumatch}
Pelin Dogan, Boyang Li, Leonid Sigal, and Markus Gross,
\newblock ``{A Neural Multi-sequence Alignment TeCHnique~(NeuMATCH)},''
\newblock in {\em IEEE Conference on Computer Vision and Pattern
  Recognition~(CVPR)}, 2018.

\bibitem{grabocka2018neuralwarp}
Josif Grabocka and Lars Schmidt-Thieme,
\newblock ``{NeuralWarp}: Time-series similarity with warping networks,''
\newblock {\em arXiv preprint arXiv:1812.08306}, 2018.

\bibitem{Petitjean2011-DBA}
Fran{\c{c}}ois Petitjean, Alain Ketterlin, and Pierre Gan{\c{c}}arski,
\newblock ``A global averaging method for dynamic time warping, with
  applications to clustering,''
\newblock {\em Pattern Recognition}, vol. 44, no. 3, pp. 678--693, 2011.

\bibitem{UCRArchive}
Yanping Chen, Eamonn Keogh, Bing Hu, Nurjahan Begum, Anthony Bagnall, Abdullah
  Mueen, and Gustavo Batista,
\newblock ``{The UCR Time Series Classification Archive},'' July 2015,
\newblock \url{www.cs.ucr.edu/~eamonn/time_series_data/}.

\bibitem{kingma2015adam}
Diederick~P Kingma and Jimmy Ba,
\newblock ``Adam: {A} method for stochastic optimization,''
\newblock in {\em International Conference on Learning Representations~(ICLR)},
  2015.

\bibitem{paszke2017automatic}
Adam Paszke, Sam Gross, Soumith Chintala, Gregory Chanan, Edward Yang, Zachary
  DeVito, Zeming Lin, Alban Desmaison, Luca Antiga, and Adam Lerer,
\newblock ``Automatic differentiation in {PyTorch},''
\newblock in {\em Neurips Autodiff Workshop}, 2017.

\bibitem{baydin2018automatic}
Atilim~Gunes Baydin, Barak~A Pearlmutter, Alexey~Andreyevich Radul, and
  Jeffrey~Mark Siskind,
\newblock ``Automatic differentiation in machine learning: a survey,''
\newblock {\em Journal of machine learning research}, vol. 18, no. 153, 2018.

\bibitem{nair2010rectified}
Vinod Nair and Geoffrey~E Hinton,
\newblock ``Rectified linear units improve restricted boltzmann machines,''
\newblock in {\em International conference on machine learning~(ICML)}, 2010,
  pp. 807--814.

\bibitem{huang2017densely}
G.~{Huang}, Z.~{Liu}, L.~v.~d. {Maaten}, and K.~Q. {Weinberger},
\newblock ``Densely connected convolutional networks,''
\newblock in {\em IEEE Conference on Computer Vision and Pattern Recognition
  (CVPR)}, July 2017, pp. 2261--2269.

\bibitem{tslearn}
Romain Tavenard, Johann Faouzi, Gilles Vandewiele, Felix Divo, Guillaume
  Androz, Chester Holtz, Marie Payne, Roman Yurchak, Marc Ru{\ss}wurm, Kushal
  Kolar, and Eli Woods,
\newblock ``tslearn: A machine learning toolkit dedicated to time-series
  data,'' 2017,
\newblock \url{https://github.com/tslearn-team/tslearn}.

\end{thebibliography}

\clearpage
\appendix
\section{Architecture of NTW}
\label{sec:nnan}

Figure~\ref{fig:architecture} illustrates the architecture of NTW employed for the experiments in Section~\ref{sec:experiments}.

\begin{figure}[h]
  \centering
  \includegraphics[width=0.6\linewidth]{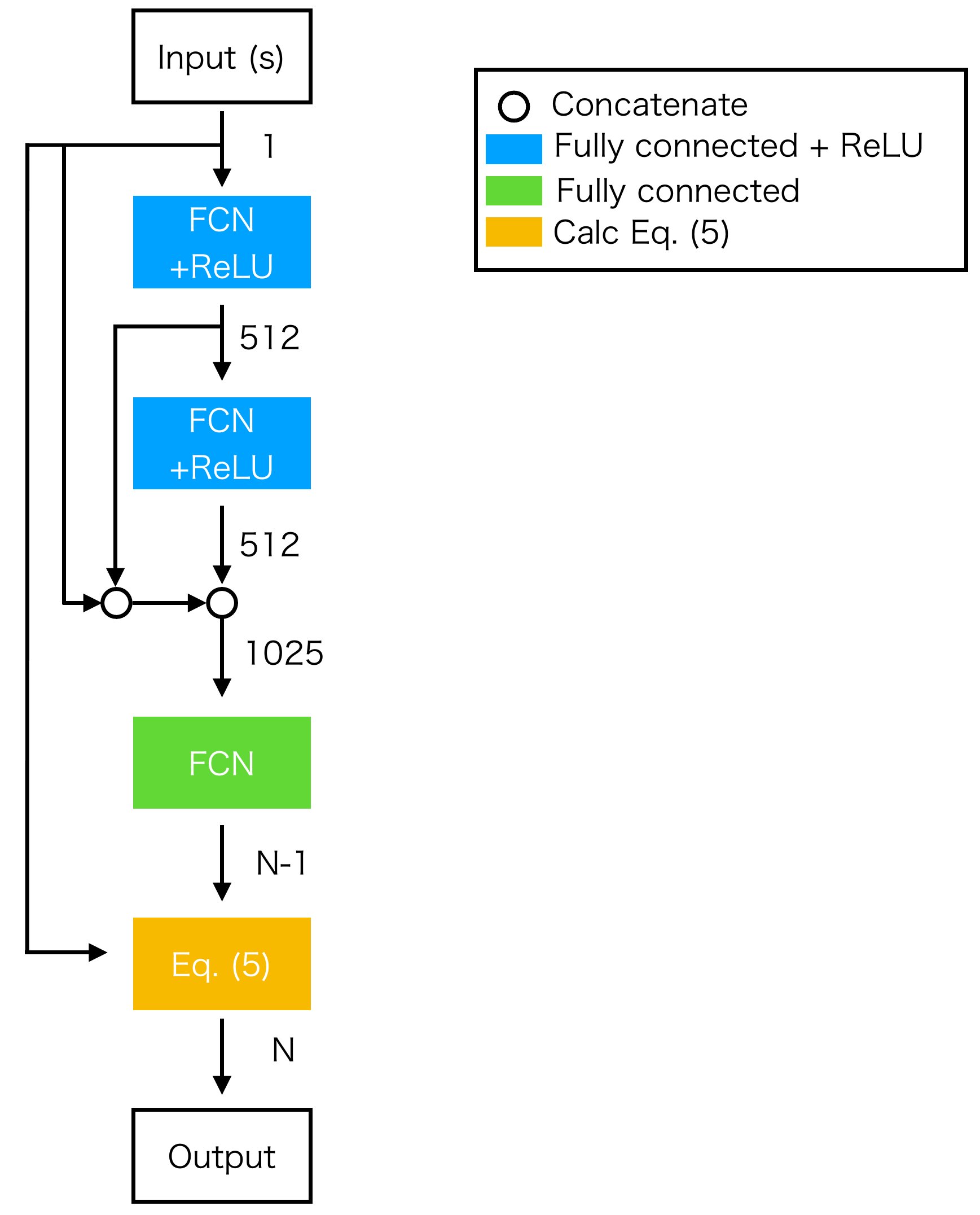}
  \caption{Architecture of NTW employed for the experiments in Section~\ref{sec:experiments}}
  \label{fig:architecture}
\end{figure}




\clearpage

\section{Pseudo-code of NTW}
\label{sec:pytorchimp}
For reproducibility, we show the PyTorch-like pseudo-code of NTW below.

\scriptsize
\begin{lstlisting}
class SincWarping(nn.Module):
    def __init__(self, warped_length, temperature,
                 temperature_multiplied=None):
        super(SincWarping, self).__init__()
        self.initial_temperature = temperature
        self.temperature = temperature
        self.temperature_multiplied = temperature_multiplied
        self.warped_length = warped_length

    def forward(self, Xs, tau):
        N, Tx, d = Xs.shape
        arange_Tx = torch.arange(Tx, device=Xs.device).float()

        ts = (tau * (Tx - 1)).unsqueeze(2).expand(N, self.warped_length, Tx) - arange_Tx
        temp = np.pi * (ts * 1 / self.temperature) + 1e-7
        sinc_ret = torch.sin(temp) / (temp)

        warped_Xs = torch.matmul(
            sinc_ret, Xs) / self.temperature

        if self.training:
            if self.temperature_multiplied is not None:
                self.temperature = max(self.temperature * self.temperature_multiplied, 1)
        return warped_Xs


class NTW(nn.Module):
    def __init__(self, N,
                 initial_temperature,
                 temperature_multiplied,
                 warped_length):
        super(NTW, self).__init__()

        self.warped_length = warped_length
        self.warp = SincWarping(warped_length, initial_temperature,
                                temperature_multiplied)
        self.net = Net(N - 1)
        d = torch.zeros(N - 1, N)
        for n in range(N - 1):
            d[n, n] = 1
        d = torch.cat([torch.ones(N).unsqueeze(0), d])
        self.bases = nn.Parameter(-torch.qr(d.t())[0], requires_grad=False)
        self.netout2fz = lambda net_outputs, sz, bases: (torch.cat([sz, net_outputs], 1) @ bases.t()).t()

    def forward(self, Xs):
        N = Xs.shape[0]
        sz = torch.linspace(0, N**0.5, self.warped_length).reshape(-1, 1).to(Xs.device)
        net_outputs = self.net(sz)
        tau = self.netout2fz(net_outputs, sz, self.bases)

        return self.warp(Xs, tau), tau

# main.py
1. # define a class of neural network for NTW
2. Xs = load_data()
3. ntw = NTW(**args)
4. while not converge
5.   warped_Xs, tau = ntw(Xs)
6.   (loss(warped_Xs) + penalty(tau)).backward()
7.   optimizer.step()
8. output warped_Xs


\end{lstlisting}

\end{document}